\documentclass{article}

     \usepackage[nonatbib,preprint]{neurips_2020}

\usepackage[utf8]{inputenc} 
\usepackage[T1]{fontenc}    
\usepackage{hyperref}       
\usepackage{url}            
\usepackage{booktabs}       
\usepackage{amsfonts}       
\usepackage{nicefrac}       
\usepackage{microtype}      
\usepackage{graphicx}
\usepackage{amsmath,amssymb}
\usepackage{amsthm}
\usepackage{nccmath}
\usepackage{here}
\usepackage{subfigure}
\usepackage[all]{xy}
\usepackage{tabularx}

\def\mcl#1{\mathcal{#1}}
\def\blacket#1{\left\langle #1\right\rangle}
\def\hil{\mcl{H}}
\def\alg{\mcl{A}}
\def\modu{\mcl{M}}
\def\blin#1{\mcl{B}(#1)}
\def\nn{\nonumber}
\def\opn{\operatorname}

\def\VVert{\vert\hspace{-1pt}\vert\hspace{-1pt}\vert}
\def\bra#1{\langle#1\vert}
\def\ket#1{\vert#1\rangle}
\def\blackett#1#2{\langle#1\vert#2\rangle}
\def\mat{\mathbb{C}^{m\times m}}
\def\bblacket#1{\big\langle #1\big\rangle}
\def\bblacketg#1{\bigg\langle #1\bigg\rangle}
\def\Bblacket#1{\bigg\langle #1\bigg\rangle}
\def\sblacket#1{\langle #1\rangle}

\def\clch{\mcl{C}_0(\mcl{X},\alg)}
\def\total{\mcl{T}(\mcl{X},\alg)}
\def\simple{\mcl{S}(\mcl{X},\alg)}
\def\bochner#1{\mcl{L}^1_{#1}(\mcl{X},\alg)}
\def\measure{\mcl{D}(\mcl{X},\alg)}
\def\bdmeasure#1{\mcl{D}_{#1}(\mcl{X},\alg)}
\def\regular{\mcl{R}_+(\mcl{X})}

\newtheorem{prop}{Proposition}[section]
\newtheorem{lemma}[prop]{Lemma}
\newtheorem{thm}[prop]{Theorem}

\newtheorem{defin}[prop]{Definition}

\newtheorem{cor}[prop]{Corollary}
\newtheorem{example}[prop]{Example}
\title{Kernel Mean Embeddings of\\Von Neumann-Algebra-Valued Measures}

\author{%
Yuka Hashimoto$^{12}$ Isao Ishikawa$^{34}$ Masahiro Ikeda$^{45}$ Fuyuta Komura$^{24}$ Yoshinobu Kawahara$^{64}$\\
$^1$NTT Network Technology Laboratories, NTT Corporation\\
$^2$Graduate School of Science and Technology, Keio University\\
$^3$Faculty of Science, Ehime University\\
$^4$Center for Advanced Intelligence Project, RIKEN\\
$^5$Faculty of Science and Technology, Keio University\\
$^6$Institute of Mathematics for Industry, Kyushu University\\
\texttt{yuka.hashimoto.rw@hco.ntt.co.jp},
\texttt{ishikawa.isao.zx@ehime-u.ac.jp},\\
\texttt{masahiro.ikeda@riken.jp},
\texttt{fuyuta.k@keio.jp}, 
\texttt{kawahara@imi.kyushu-u.ac.jp}
}

\begin{document}

\maketitle

\begin{abstract}
Kernel mean embedding (KME) is a powerful tool to analyze probability measures for data, where the measures are conventionally embedded into a reproducing kernel Hilbert space (RKHS).
In this paper, we generalize KME to that of von Neumann-algebra-valued measures into reproducing kernel Hilbert modules (RKHMs), which provides an inner product and distance between von Neumann-algebra-valued measures.
Von Neumann-algebra-valued measures can, for example, encode relations between arbitrary pairs of variables in a multivariate distribution or positive operator-valued measures for quantum mechanics.
Thus, this allows us to perform probabilistic analyses explicitly reflected with higher-order interactions among variables, and provides a way of applying machine learning frameworks to problems in quantum mechanics.
We also show that the injectivity of the existing KME and the universality of RKHS are generalized to RKHM, which confirms many useful features of the existing KME remain in our generalized KME.
And, we investigate the empirical performance of our methods using synthetic and real-world data.
\end{abstract}

\section{Introduction}\label{sec:intro}
Kernel mean embedding (KME) is a powerful tool to analyze probability distributions (or measures) for data, where each distribution is conventionally embedded as a function in a reproducing kernel Hilbert space (RKHS)~\cite{solma07,kernelmean,sriperumbudur11}.
Since an RKHS has an inner product, it provides a distance between two distributions, which is used in, for example, statistical tests for comparing samples from two  distributions~\cite{gretton05,gretton12,jitkrittum17}, and the development of various learning algorithms~\cite{song10,jitkrittum19,Li19}.
As is well known, KME has superior features both from the aspects of representation and computation.
For example, 
an injective KME
can encode any distribution (any finite real-valued signed measure) into a vector in an RKHS~\cite{fukumizu08,sriperumbudur10,sriperumbudur11}. 
Meanwhile, from 
the reproducing property of RKHS, computations in RKHSs are explicitly performed even though the dimension of RKHSs is essentially infinite. 

However, embedding into RKHSs can be ineffective for multivariate data because inner products between two vectors in RKHSs are real or complex-valued, which is not adequate for describing the relation of each pair in variables.
More precisely, similarities between arbitrary pairs of variables in a distribution are degenerated into one complex or real value; therefore, it is difficult to discriminate the information of these similarities from the corresponding inner products.

In this paper, we apply theories of von Neumann-algebra-valued measures (more generally, vector-valued measures) 
to define KME of von Neumann-algebra valued measures, and generalize the KME in RKHS to reproducing kernel Hilbert modules (RKHMs), which enables us to embed von Neumann-algebra-valued measures into RKHMs. 
RKHM is a generalization of RKHS~\cite{itoh90,heo08,szafraniec10,hashimoto20}, and von Neumann-algebra is a special class of bounded linear operators on a Hilbert space. 
An important example of von Neumann-algebras is the space of matrices $\mat$, where a $\mat$-valued measure can describe $m^2$ variables simultaneously; thus, it can be employed to describe relations of variable pairs in the 
distributions. 
Since RKHSs are too small to represent von Neumann-algebra-valued measures, we use RKHMs instead of RKHSs.
That is, whereas an RKHS is composed of complex-valued functions, an RKHM is composed of von Neumann-algebra valued functions, which has sufficient representation power for von Neumann-algebra valued measures.




We provide sufficient conditions of the injectivity of the proposed KME and derive a connection between the injectivity and universality of RKHM.  
As a result, RKHMs associated with well-known kernels, such as the Gaussian and Laplacian kernels, are shown to have both injectivity and universality.
The injectivity of KMEs is important for regarding any measure as a vector in an RKHM.
In addition, universality is also relevant to ensure kernel-based models approximate any continuous target function arbitrarily well.
For RKHS, these two properties are related and have been actively studied to theoretically guarantee the validity of kernel-based algorithms~\cite{steinwart01,gretton07,fukumizu08,sriperumbudur11}. 
However, to the best of our knowledge, necessary and sufficient conditions for injectivity and universality, and the connection between them have not been known so far for RKHM.

Furthermore, we apply the proposed KME to practical examples of von Neumann-algebra-valued measures.
One example is a $\mat$-valued measure that encodes relations between arbitrary pairs of $m$ variables in a multivariate distribution, which allows us to perform data analyses explicitly reflected with high-order interactions between variables.
Another important example is a positive operator-valued measure, often considered in quantum mechanics.
Recently, applying machine learning to problems for quantum mechanics, such as quantum tomography and anomaly detection of quantum state, has been actively studied~\cite{torlai18,srinivasan19,cranmer19,hara16,liu18}, where complex-valued inner products between two quantum states are often employed~\cite{balkir14,prasenjit16,liu18}. 
We show that our proposed KME generalizes many existing methods for the above measures.

The remainder of this paper is organized as follows:
First, in Section~\ref{sec:background}, we briefly review the theory of RKHM and von Neumann-algebra-valued measure. 
In Section~\ref{sec:kme}, we define the KMEs of von Neumann-algebra-valued measures into RKHMs. 
In Section~\ref{sec:properties}, we provide the sufficient conditions for the injectivity of our KME and derive connections of injectivity to universality.
Moreover, we discuss two specific cases of von Neumann-algebra-valued measures in Section~\ref{sec:example}, then propose a generalized maximum mean discrepancy (MMD) and kernel principal component analysis (PCA) for von Neumann-algebra-valued measures using our KME in Section~\ref{sec:application}.
And finally, we 
conclude the paper in Section~\ref{sec:conclusion}. 
The notations in this paper are explained in Appendix~\ref{ap:notation}
and proofs are given in Appendices~\ref{ap:universal} and \ref{sec:proof} in the supplementary material.

\section{Background}\label{sec:background}
In this section, we review von Neumann-algebra and its module in Subsection~\ref{sec:c_star_alg}, RKHM in Subsection~\ref{sec:rkhm_review}, and von Neumann-algebra-valued measure in Subsection~\ref{sec:vv-measure}.

\subsection{Von Neumann-algebra and module}\label{sec:c_star_alg}
A von Neumann-algebra and module are suitable generalizations of the space of complex numbers $\mathbb{C}$ and a vector space, respectively~\cite{lance95}.  
As we see below, many complex-valued notions can be generalized to von Neumann-algebra-valued. 

Let $\blin{\hil}$ be the set of all bounded linear maps on a Hilbert space $\hil$, equipped with the operator norm $\Vert\cdot\Vert_{\alg}$. 
A {\em von Neumann-algebra} $\alg$ is defined as a subspace of $\blin{\hil}$ which  
is closed with respect to the strong operator topology
(i.e., $c\in\alg$ if and only if there exists $\{c_i\}_i\subseteq\alg$ such that $\lim_{i\to\infty}\Vert c_ih-ch\Vert_{\hil}=0$ for all $h\in\hil$), and equipped with a product structure and an involution.
For example, $\blin{\hil}$ and $L_{\mu}^{\infty}(\mcl{X})$ for a measurable space $\mcl{X}$ and $\sigma$-finite measure $\mu$ are von Neumann-algebras.
If $\hil$ is finite dimensional, $\blin{\hil}$ is the space of matrices.

An {\em $\alg$-module} $\mathcal{M}$ 
is a linear space equipped with a {\em right $\mathcal{A}$-multiplication}.  
For $u\in \mathcal{M}$ and $c\in\mathcal{A}$, the right $\alg$-multiplication of $u$ with $c$ is denoted as $uc$. 
If $\modu$ is equipped with an $\alg$-valued inner product and complete, it is called a {\em Hilbert $\alg$-module}.
Here, an {\em $\alg$-valued inner product} is a map $\blacket{\cdot,\cdot}\colon \modu\times\modu\to\alg$ that satisfies the following four conditions for $u,v,w\in\modu$ and $c,d\in\alg$:
1.\@ $\blacket{u,vc+wd}=\blacket{u,v}c+\blacket{u,w}d$,
2.\@ $\blacket{v,u}=\blacket{u,v}^*$,
3.\@ $\blacket{u,u}\ge 0$, and 4.\@ $\blacket{u,u}=0$ implies $u=0$. 
The $\alg$-valued inner product induces the notion of orthonormal, which is important for solving various minimization problems~\cite{hashimoto20}.
A set of vectors $\{p_1,\ldots,p_s\}\subseteq\modu$ is called an {\em orthonormal system (ONS)} if $\blacket{p_i,p_j}=0$ for $i\neq j$ and $\blacket{p_i,p_i}$ is a nonzero projection operator.

Another important feature of Hilbert $\alg$-module is the Riesz representation theorem~\cite[Theorem 4.16]{skeide00}.
For RKHS, the Riesz representation theorem is necessary to define the KME.
We will also use this type of theorem for modules to define a KME in an RKHM.
\begin{thm}[The Riesz representation theorem for Hilbert $\alg$-modules]\label{thm:riesz}
 Let $\alg$ be a von Neumann algebra.
 For a bounded $\alg$-linear map $L:\modu\to\alg\;$, there exists a unique $u\in\modu$ such that $Lv=\blacket{u,v}$ for all $v\in\modu$.
 Here, an $\alg$-linear map $L:\modu\to\alg\;$ is defined as a linear map that satisfies $L(vc)=(Lv)c\;$ for any $u\in\modu$ and $c\in\alg$.
\end{thm}

\subsection{Reproducing kernel Hilbert module (RKHM)}\label{sec:rkhm_review}
An RKHM is a Hilbert $\alg$-module composed of $\alg$-valued functions on a non-empty set $\mcl{X}$.
Let $k:\mcl{X}\times \mcl{X}\to\mcl{A}$ be an $\alg$-valued positive definite kernel on $\mcl{X}$, i.e., it satisfies the following:

1. $k(x,y)=k(y,x)^*$ for $x,y\in\mcl{X}$, and
2. $\sum_{i,j=1}^nc_i^*k(x_i,x_j)c_j$ is positive semi-definite for $x_i\in\mcl{X}$ and $c_i\in\alg$.

Let $\phi:\mcl{X}\to\alg^{\mcl{X}}$ be the {\em feature map} associated with $k$, which is defined as $\phi(x)=k(\cdot,x)$ for $x\in\mcl{X}$.
We construct the $\alg$-module composed of all finite sums $\sum_{i}\phi(x_i)c_i$ 
and define an $\alg$-valued inner product $\blacket{\cdot,\cdot}_k:\mcl{X}\times \mcl{X}\to\alg$ through $k$ as
\begin{equation*}
\bblacketg{\sum_{s=1}^{n}\phi(x_s)c_s,\sum_{t=1}^{l}\phi(y_t)d_t}_{k}:=\sum_{s=1}^{n}\sum_{t=1}^{l}c_s^*k(x_s,y_t)d_t.
\end{equation*}
The completion of this $\alg$-module is called a {\em reproducing kernel Hilbert $\alg$-module (RKHM)} associated with $k$ and denoted as $\modu_k$.  
An RKHM has the reproducing property, i.e., 
\begin{equation}
\blacket{\phi(x),u}_k=u(x),\label{eq:reproducing}
\end{equation}
for $u\in\modu_k$ and $x\in\mcl{X}$.
Also, we define the {\em $\alg$-valued absolute value} $\vert u\vert_k$ on $\modu_k$ by the positive semi-definite element $\vert u \vert_k$ of $\alg$ such that $\vert u\vert_k^2=\blacket{u,u}_k$. 
In the following, we drop subscript $k$ in $\blacket{\cdot,\cdot}_k$ and $\vert\cdot\vert_k$ to simplify the notation.

\subsection{$\alg$-valued measure and integral}\label{sec:vv-measure}
The notions of measures and the Lebesgue integrals are generalized to $\alg$-valued.
The {\em left and right integral of an $\alg$-valued function $u$ with respect to an $\alg$-valued measure $\mu$} is defined through $\alg$-valued step functions.
They are respectively denoted as 
\begin{equation*}
\int_{x\in\mcl{X}}d\mu(x)u(x)\in\alg\quad\text{and}\quad \int_{x\in\mcl{X}}u(x)d\mu(x)\in\alg.
\end{equation*}
Note that since the multiplication in $\alg$ is not commutative in general, the left and right integrals do not always coincide.
For further details about $\alg$-valued measure and its integral, see Appendix~\ref{ap:review_vvmeasure}.


\section{Kernel mean embedding of $\alg$-valued measures}\label{sec:kme}
In this section, we propose a generalization of the existing KME in RKHS \cite{kernelmean,sriperumbudur11}
to an embedding of $\alg$-valued measures into an RKHM.

Let $\mcl{X}$ be a locally compact Hausdorff space for data.
We often consider $\mcl{X}=\mathbb{R}^d$ in practical situations.
Let $\alg$ be a von Neumann-algebra, $\mcl{D}(\mcl{X},\alg)$ be the set of all $\alg$-valued finite regular Borel measures, and
$\clch$ be the set of all continuous $\alg$-valued functions on $\mcl{X}$ vanishing at infinity.
Note that if $\mcl{X}$ is compact, any continuous function is contained in $\clch$.
In addition, let $k$ be an $\alg$-valued $c_0$-kernel, i.e., 
$k$ is bounded and $\phi(x)\in\clch$ for any $x\in\mcl{X}$ (e.g., a diagonal matrix-valued kernel whose elements are Gaussian, Laplacian or $B_{2n+1}$-spline kernel, see Appendix~\ref{ap:kernel}).
We now define a KME in an RKHM as follows:
\begin{defin}[KME in RKHM]
A {\em kernel mean embedding} in an RKHM $\modu_k$ is a map $\Phi: \mcl{D}(\mcl{X},\alg)\rightarrow \modu_k$ defined by
\begin{equation}
 \Phi(\mu):=\int_{x\in\mcl{X}}\phi(x)d\mu(x).\label{eq:km_rkhm}
\end{equation}
\end{defin}
We emphasize that the well-definedness of $\Phi$ is not trivial, and von Neumann-algebras are adequate to show it. More precisely, the following theorem derives the well-definedness:
\begin{thm}[Well-definedness for the KME in RKHM]\label{thm:kme}
Let $\mu\in\mcl{D}(\mcl{X},\alg)$.
Then, $\Phi(\mu)\in\modu_k$. 
In addition, the following equality holds for any $v\in\modu_k$:
\begin{equation}
\blacket{\Phi(\mu),v}=\int_{x\in\mcl{X}}d\mu^*(x)v(x).\label{eq:reproducing_km}
\end{equation}
\end{thm}
\begin{cor}
 For $\mu,\nu\in\mcl{D}(\mcl{X},\alg)$, the inner product between $\Phi(\mu)$ and $\Phi(\nu)$ is given as follows:
\begin{equation*}
\blacket{\Phi(\mu),\Phi(\nu)}=\int_{x\in\mcl{X}}\int_{y\in\mcl{X}}d\mu^*(x)k(x,y)d\nu(y).
\end{equation*}
\end{cor}
Moreover, many basic properties for the existing KME in RKHS are generalized to the proposed KME as follows:
\begin{prop}
For $\mu,\nu\in\mcl{D}(\mcl{X},\alg)$ and $c\in\alg$, $\Phi(\mu+\nu)=\Phi(\mu)+\Phi(\nu)$ and $\Phi(\mu c)=\Phi(\mu)c$ hold.
In addition, for $x\in\mcl{X}$, let $\delta_x$ be the $\alg$-valued Dirac measure defined as $\delta_x(E)=1_{\alg}$ for $x\in E$ and $\delta_x(E)=0$ for $x\notin E$.
Then, $\Phi(\delta_x)=\phi(x)$.
\end{prop}
This is derived from Eqs.~\eqref{eq:km_rkhm} and \eqref{eq:reproducing_km}.
Note that if $\alg=\mathbb{C}$, then the proposed KME~\eqref{eq:km_rkhm} is equivalent to the existing KME in RKHS defined in~\cite{sriperumbudur11}.

\section{Injectivity and universality of the proposed KME}\label{sec:properties}
In this section, we generalize the injectivity of the existing KME and universality of RKHS to those of the proposed KME for von Neumann-algebra-valued measures and RKHM, respectively.
For RKHS, injectivity and universality have been actively researched for guaranteeing the effectiveness of kernel-based algorithms.
In Subsection~\ref{sec:characteristic}, we provide sufficient conditions for the injectivity of the proposed KME. 
Then, in Subsection~\ref{sec:universal}, we derive a connection of injectivity to universality.

\subsection{Injectivity}\label{sec:characteristic}
In practice, the injectivity of $\Phi$ is important to transform problems in $\mcl{D}(\mcl{X},\alg)$ into those in $\modu_k$.
This is because if a KME $\Phi$ in an RKHM is injective, then
$\alg$-valued measures are embedded into $\modu_k$ through $\Phi$ without loss of information.
Note that, for probability measures, the injectivity of the existing KME is also referred to as the ``characteristic'' property.
The injectivity of the existing KME in RKHS has been discussed in, for example, \cite{fukumizu08,sriperumbudur10,sriperumbudur11}.  
These studies give criteria for the injectivity of the KMEs associated with important complex-valued kernels such as transition invariant kernels and radial kernels.
Typical examples of these kernels are Gaussian, Laplacian, and inverse multiquadratic kernels.
Here, we define the transition invariant kernels and radial kernels for $\alg$-valued measures, and generalize their criteria to RKHMs associated with $\alg$-valued kernels. 

Let $\hat{\lambda}$ be the Fourier transform of an $\alg$-valued measure $\lambda$ defined as $\hat{\lambda}=\int_{\omega\in\mathbb{R}^d}e^{-\sqrt{-1}x^T\omega}d\lambda(\omega)$, and $\opn{supp}(\lambda):=\{x\in\mathbb{R}^d\mid\ \mbox{for any open set $U$ such that }x\in U,\ \lambda(U)\mbox{ is positive definite}\}$. 
An $\alg$-valued positive definite kernel $k:\mcl{X}\times\mcl{X}\to\alg$ is called a {\em transition invariant kernel} if it is represented as $k(x,y)=\hat{\lambda}(y-x)$ for a positive semi-definite $\alg$-valued measure $\lambda$.
In addition, $k$ is called a {\em radial kernel} if it is represented as $k(x,y)=\int_{[0,\infty)}e^{-t\Vert x-y\Vert^2}d\eta(t)$ for a positive semi-definite $\alg$-valued measure $\eta$.
\begin{thm}\label{thm:characteristic}
 Let $\alg=\mat$ and $\mcl{X}=\mathbb{R}^d$.
 Assume $k:\mcl{X}\times\mcl{X}\to\alg$ is a transition invariant kernel with a positive semi-definite $\alg$-valued measure $\lambda$ that satisfies $\opn{supp}(\lambda)=\mcl{X}$.
Then, 
KME $\Phi:\mcl{D}(\mcl{X},\alg)\to\modu_k$ defined as Eq.~\eqref{eq:km_rkhm} is injective.
\end{thm}

\begin{thm}\label{thm:characteristic2}
 Let $\alg=\mat$ and $\mcl{X}=\mathbb{R}^d$.
 Assume $k:\mcl{X}\times\mcl{X}\to\alg$ is a radial kernel with a positive definite $\alg$-valued measure $\eta$ that satisfies $\opn{supp}(\eta)\neq \{0\}$.
 Then, KME $\Phi:\mcl{D}(\mcl{X},\alg)\to\modu_k$ defined as Eq.~\eqref{eq:km_rkhm} is injective.
\end{thm}

\begin{example}
If $k$ is a diagonal matrix-valued kernel whose diagonal elements are Gaussian, Laplacian, or $B_{2n+1}$-spline
, then $k$ is a $c_0$-kernel 
(Example~\ref{ex:pdk1}).
There exists a diagonal matrix-valued measure $\lambda$ that satisfies $k(x,y)=\hat{\lambda}(y-x)$ and whose diagonal elements are nonnegative and supported by $\mathbb{R}^d$ (c.f. Table 2 in~\cite{sriperumbudur10}).
Thus, by Theorem~\ref{thm:characteristic}, $\Phi$ is injective.
\end{example}

\begin{example}
If $k$ is a diagonal matrix-valued kernel whose diagonal elements are inverse multiquadratic
, then $k$ is a $c_0$-kernel 
(Example~\ref{ex:pdk1}).
There exists a diagonal matrix-valued measure $\eta$ that satisfies $k(x,y)=\int_{[0,\infty)}e^{-t\Vert x-y\Vert^2}d\eta(t)$, and whose diagonal elements are nonnegative and $\opn{supp}(\eta)\neq\{0\}$  
(c.f. Theorem 7.15 in~\cite{wendland04}).
Thus, by Theorem~\ref{thm:characteristic2}, $\Phi$ is injective.
\end{example}


%
\subsection{Connection of injectivity with universality}\label{sec:universal}
Another important property for kernel methods is universality, which ensures that
kernel-based algorithms approximate 
each continuous target function arbitrarily well.
For RKHS, Sriperumbudur~\cite{sriperumbudur11} showed the equivalence of the injectivity of the existing KME and universality.
Mathematically, an RKHS (or RKHM in our case) is said to be universal if it is dense in a space of bounded and continuous functions.
We show the above equivalence holds also for RKHM.
\begin{thm}\label{thm:universal_finitedim}
Let $\alg=\mat$. Then, $\Phi:\measure\to\modu_k$ is injective if and only if $\modu_k$ is dense in $\clch$.
\end{thm}
By Theorem~\ref{thm:universal_finitedim}, if $k$ satisfies the condition in Theorem~\ref{thm:characteristic} 
or \ref{thm:characteristic2}, 
then $\modu_k$ is universal.

For the case where $\alg$ is infinite dimensional, the universality of $\modu_k$ in $\clch$ is a sufficient condition for the injectivity of the proposed KME,
although the equivalence of the injectivity of the KME and universality is an open problem.
\begin{thm}\label{thm:universal}
If $\modu_k$ is dense in $\clch$, $\Phi:\measure\to\modu_k$ is injective.
\end{thm}
The details of the derivations of Theorems~\ref{thm:universal_finitedim} and \ref{thm:universal} are given in  Appendix~\ref{ap:universal}.

\section{Specific examples of $\alg$-valued measures and their KME}\label{sec:example}
Here, we give two important examples of $\alg$-valued measures and show the proposed KME of these measures generalizes existing notions.
In Subsection~\ref{sec:covariance}, we propose a cross-covariance measure, which encodes the relation between arbitrary pairs of variables in distributions. 
In Subsection~\ref{sec:povm}, we consider the positive operator-valued measure, which plays an important role in quantum mechanics.

\subsection{Cross-covariance measure}\label{sec:covariance}
We propose a matrix-valued measure that encodes the relation between arbitrary pairs of $m$ random variables into an $m\times m$ symmetric matrix.
Let $(\Omega,\mcl{F})$ be a measurable space, $P$ be a real-valued probability measure on $\Omega$, and $X_1,\ldots,X_m,Y_1,\ldots,Y_m:\Omega\to\mcl{X}$ be random variables.
In addition, let $\alg=\mat$ and $k:\mcl{X}^2\times \mcl{X}^2\to\alg$ be an $\alg$-valued positive definite kernel.
\begin{defin}[Cross-covariance measure]\label{def:covariance}
For a Borel set $E$, we define a (uncentered) {\em cross-covariance measure} $\mu_X\in\mcl{D}(\mcl{X},\mat)$ of $X=[X_1,\ldots,X_m]$ as
\begin{equation*}
[\mu_X(E)]_{i,j}=(X_i,X_j)_*P(E),
\end{equation*}
where $X_*P$ means the push forward measure of $P$ with respect to a random variable $X$.
We also define the centered version of $\mu_X$ as $[\tilde{\mu}_X]_{i,j}=[\mu_X]_{i,j}-{X_i}_*P\otimes{X_j}_*P$. 
\end{defin}

We show that a discrepancy between $\Phi(\mu_X)$ and $\Phi(\mu_Y)$ is equal to that between operators composed of cross-covariance operators for a specific $\alg$-valued positive definite kernel.
Thus, our KME of the cross-covariance measure generalizes the notion of cross-covariance operator.
The cross-covariance operator is a generalization of the cross-covariance matrix~\cite{baker73,fukumizu04,kernelmean}. 
It is a linear operator $\Sigma_{X_i,X_j}:\hil_{\tilde{k}_1}\to\hil_{\tilde{k}_2}$ defined as $\Sigma_{X_i,X_j}:=\int_{\omega\in\Omega}\tilde{\phi}_1(X_i(\omega))\otimes \tilde{\phi}_2(X_j(\omega))dP(\omega)$,
where $\tilde{k}_1$ and $\tilde{k}_2$ are complex-valued positive definite kernels, $\tilde{\phi}_1$ and $\tilde{\phi}_2$ are their feature maps, and $\hil_{\tilde{k}_1}$ and $\hil_{\tilde{k}_2}$ are the RKHSs associated with $\tilde{k}_1$ and $\tilde{k}_2$, respectively.
\begin{thm}\label{prop:cross-covariance}
Assume $k(x,y)=\tilde{k}(x,y)I$, where $\tilde{k}((x_1,x_2),(y_1,y_2))=\tilde{k}_1(x_1,y_1)\tilde{k}_2(x_2,y_2)$ for $x=(x_1,x_2),y=(y_1,y_2)\in\mcl{X}^2$ and $I$ is the identity matrix.
Then, $\opn{tr}(\vert \Phi(\mu_X)-\Phi(\mu_Y)\vert^2)=\Vert \Sigma_X-\Sigma_Y\Vert_{\opn{HS}}^2$ holds, where $\Sigma_X=[\Sigma_{X_i,X_j}]_{i,j}$, and $\Vert \cdot\Vert_{\opn{HS}}$ is the Hilbert-Schmidt norm.
\end{thm}

\subsection{Positive operator-valued measure}\label{sec:povm}
A positive operator-valued measure 
is defined as an $\alg$-valued measure $\mu$ such that $\mu(\mcl{X})=I$ and $\mu(E)$ is positive semi-definite for any Borel set $E$.
It enables us to extract information of the probabilities of outcomes from a state~\cite{peres04,holevo11}.
We show that the existing inner product considered for quantum states~\cite{balkir14,prasenjit16} is generalized with our KME of positive operator-valued measures. 

Let $\mcl{X}=\mathbb{C}^m$ and $\alg=\mat$.
Let $\rho\in\alg$ be a positive semi-definite matrix with unit trace, called a density matrix.  
A density matrix describes the states of a quantum system,  
and information about outcomes is described as measure $\mu\rho\in\mcl{D}(\mcl{X},\alg)$.
We have the following theorem.
Here, we use the bra-ket notation, i.e., 
$\ket{\alpha}\in\mcl{X}$ represents a (column) vector in $\mcl{X}$, and $\bra{\alpha}$ is defined as $\bra{\alpha}:=\ket{\alpha}^{*}$:
\begin{thm}\label{prop:inprod_equiv}
Assume $\mcl{X}=\mathbb{C}^m$, $\alg=\mat$, and $k:\mcl{X}\times\mcl{X}\to\alg$ is a positive definite kernel defined as $k(\ket{\alpha},\ket{\beta})=\ket{\alpha}\blackett{\alpha}{\beta}\bra{\beta}$.
If $\mu$ is represented as $\mu=\sum_{i=1}^m\delta_{\ket{\psi_i}}\ket{\psi_i}\bra{\psi_i}$ for an orthonormal basis $\{\ket{\psi_1},\ldots,\ket{\psi_m}\}$ of $\mcl{X}$, then  $\opn{tr}(\blacket{\Phi(\mu\rho_1),\Phi(\mu\rho_2)})=\blacket{\rho_1,\rho_2}_{\opn{HS}}$ holds. 
Here, $\blacket{\cdot,\cdot}_{\opn{HS}}$ is the Hilbert--Schmidt inner product.
\end{thm}
In previous studies~\cite{balkir14,prasenjit16}, the Hilbert--Schmidt inner product between density matrices was considered to represent similarities between two quantum states. 
Liu et al.~\cite{liu18} considered the Hilbert--Schmidt inner product between square roots of density matrices.
Theorem~\ref{prop:inprod_equiv} shows that these inner products are represented via our KME in RKHM.

\section{Applications}\label{sec:application}
In this section, we provide several applications of the proposed KME.
We introduce an MMD for $\alg$-valued measures in Subsection~\ref{sec:mmd} and a kernel PCA for $\alg$-valued measures in Subsection~\ref{sec:pca}.
We then mention other applications in Section~\ref{sec:other}.

\subsection{Maximum mean discrepancy with kernel mean embedding}\label{sec:mmd}
MMD is a metric of measures according to the largest difference in means over a certain subset of a function space. 
It is also known as integral probability metric (IPM). 
For a set $\mcl{U}$ of real-valued functions on $\mcl{X}$ and two real-valued probability measures $\mu$ and $\nu$,
MMD $\gamma(\mu,\nu,\mcl{U})$ is defined as $\sup_{u\in\mcl{U}}\big\vert\int_{x\in\mcl{X}}u(x)d\mu(x)-\int_{x\in\mcl{X}}u(x)d\nu(x)\big\vert$~\cite{muller97,gretton12}.
For example, if $\mcl{U}$ is the unit ball of an RKHS, denoted as $\mcl{U}_{\opn{RKHS}}$, the MMD can be represented using the KME $\tilde{\Phi}$ in the RKHS as $\gamma(\mu,\nu,\mcl{U}_{\opn{RKHS}})=\Vert\tilde{\Phi}(\mu)-\tilde{\Phi}(\nu)\Vert$.
Let $\mcl{U}_{\alg}$ be a set of $\alg$-valued bounded and measurable functions and $\mu,\nu\in\mcl{D}(X,\alg)$.
We generalize the MMD to that for $\alg$-valued measures as follows:
\begin{equation*}
 \gamma_{\alg}(\mu,\nu,\mcl{U}_{\alg}):=\sup_{u\in\mcl{U}}\bigg\vert\int_{x\in\mcl{X}}u(x)d\mu(x)-\int_{x\in\mcl{X}}u(x)d\nu(x)\bigg\vert_{\alg},
\end{equation*}
where $\vert c\vert_{\alg}:=(c^*c)^{1/2}$ for $c\in\alg$ 
and supremum is taken with respect to a (pre) order in $\alg$ (see Appendix~\ref{ap:notation} for further details). 
The following theorem shows that similar to the case of RKHS, if $\mcl{U}_{\alg}$ is the unit ball of an RKHM, the generalized MMD $\gamma_{\alg}(\mu,\nu,\mcl{U}_{\alg})$ can also be represented using the proposed KME in the RKHM.

\begin{prop}\label{prop:mmd}
Let $\mcl{U}_{\opn{RKHM}}:=\{u\in\modu_k\mid\ \Vert u\Vert\le 1\}$.
Then, for $\mu,\nu\in\mcl{D}(\mcl{X},\alg)$, we have
 \begin{equation*}
\gamma_{\alg}(\mu,\nu,\mcl{U}_{\opn{RKHM}})=\vert \Phi(\mu)-\Phi(\nu)\vert.
\end{equation*}
\end{prop}

Various methods with the existing MMD of real-valued probability measures are generalized to $\alg$-valued measures by applying our MMD. 
An example of $\alg$-valued measures is the cross-covariance measure defined in Definition~\ref{def:covariance}.
Using our MMD of the cross-covariance measures instead of the existing MMD allows us to encode higher-order interactions among variables in the methods.
For example, the following existing methods can be generalized:

{\bf Two-sample test:}\quad In two-sample test, samples from two distributions (measures) are compared by computing the MMD of these measures~\cite{gretton12}. 

{\bf Kernel mean matching for generative models:}\quad In generative models, MMD is used in finding points whose distribution is as close as to that of input points~\cite{jitkrittum19}. 

{\bf Domain adaptation:}\quad In Domain adaptation, MMD is used in describing distributions of target domain data as close to those of source domain data~\cite{Li19}. 

\subsubsection*{Numerical results}
We applied our MMD of cross-covariance measures defined in Subsection~\ref{sec:covariance} to two-sample test by using real-world climate data in Canada on January 2020\footnote{Available at \url{https://climate.weather.gc.ca/prods_servs/cdn_climate_summary_e.html}}.
We compared the two types of samples, each of which is composed of three variables representing
(a) longitude, latitude, and temperature or (b) longitude, latitude, and precipitation. 
We prepared two sample sets by randomly selecting $N$ samples from (a) or (b) as follows:
Case 1: (a) and (a), 
Case 2: (b) and (b), and
Case 3: (a) and (b). 
Both (a) and (b) contain 277 samples.
The experiments were implemented with Python 3.7.

Let $(\Omega,\mcl{F},P)$ be a probability space.
Assume samples in (a) are generated by a random variable $X=[X_1,X_2,X_3]:\Omega\to\mcl{X}^3$ and samples in (b) are generated by a random variable $Y=[Y_1,Y_2,Y_3]:\Omega\to\mcl{X}^3$.
Let $\mu_X$ and $\mu_Y$ be the cross-covariance measures with respect to $X$ and $Y$ defined in Definition~\ref{def:covariance}.
We computed the norm of our MMD $\gamma_{\alg}(\mu_X,\mu_Y,\mcl{U}_{\opn{RKHM}})$.
For comparison, MMDs $\gamma(X_*P,Y_*P,\mcl{U})$ with $\mcl{U}_{\opn{RKHS}}$, $\mcl{U}_{\opn{K}}$, and $\mcl{U}_{\opn{D}}$ were also computed.
Here, $\mcl{U}_{\opn{K}}:=\{u\mid\ \Vert u\Vert_L\le 1\}$, and $\mcl{U}_{\opn{D}}:=\{u\mid\ \Vert u\Vert_{\infty}+\Vert u\Vert_L\le 1\}$, where, $\Vert u\Vert_L:=\sup_{x\neq y}\vert u(x)-u(y)\vert/\vert x-y\vert$, and  $\Vert u\Vert_{\infty}$ is the sup norm of $u$. 
The MMDs with $\mcl{U}_{\opn{K}}$ and $\mcl{U}_{\opn{D}}$ are discussed in~\cite{ratchev85,dudley02,sriperumbudur12}.
We used Bootstrap to estimate the $1-\alpha$ quantiles of the distributions of the MMDs under a null hypothesis $\mu_X=\mu_Y$ or $X_*P=Y_*P$. 
Figure~\ref{fig:sampletest} illustrates the acceptance rate  
of the null hypothesis in 100 repetitions with $\alpha=0.05$ in the case of $N=10,20,30,50,100$.
We used $\mat$-valued kernel $k(x,y)=e^{-\Vert x-y\Vert^2}I$, where $m=3$ for $\mcl{U}=\mcl{U}_{\opn{RKHM}}$, and complex-valued kernel $\tilde{k}(x,y)=e^{-\Vert x-y\Vert^2}$ for $\mcl{U}=\mcl{U}_{\opn{RKHS}}$.
Note that as we mentioned in Section~\ref{sec:characteristic}, both KMEs associated with the above $k$ and $\tilde{k}$ are injective.
We can see our MMD of cross-covariance measures with respect to RKHM attains 
a higher acceptance rate for Cases 1 and 2 (both samples are from the same type of data),
and a lower acceptance rate for Case 3 (two samples are from different types of data).

\begin{figure}[t]\label{fig:sampletest}
 \begin{center}
  \includegraphics[scale=0.28]{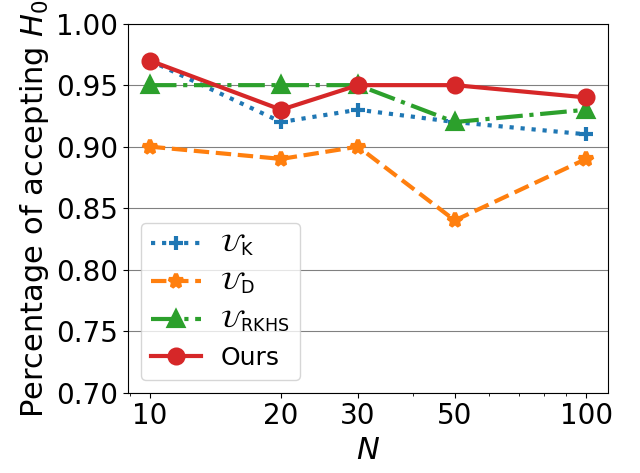}
  \includegraphics[scale=0.28]{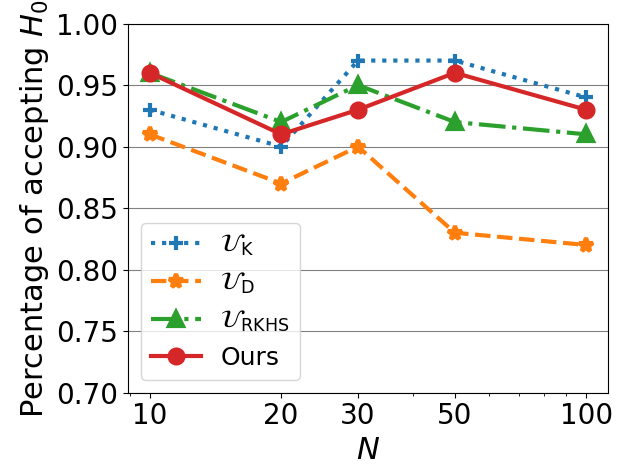}
  \includegraphics[scale=0.28]{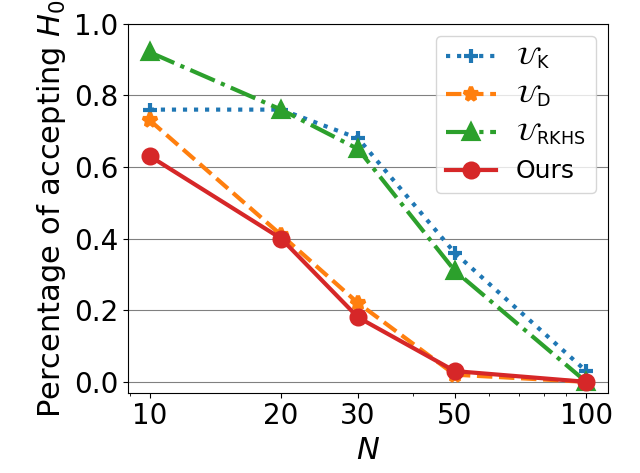}
  \caption{\small 
   Acceptance rate  
of null hypothesis $\mu_X=\mu_Y$ or $X_*P=Y_*P$, denoted as $H_0$, in 100 repetitions with $\alpha=0.05$ (Left: Case~1, Middle: Case~2, and Right: Case~3). Note that $H_0$ holds for Case~1 and Case~2, but not for Case~3.}
 \end{center}
\end{figure}

\subsection{Kernel PCA for matrix-valued measures}\label{sec:pca}
PCA is a statistical procedure to find a low dimensional subspace that preserves information of samples, which has been applied to, for example, visualization and anomaly detection~\cite{jolliffe02,liu18,hashimoto20}.
We introduce a PCA for $\alg$-valued measures in terms of the proposed KME in RKHM.
Let $\alg=\mat$ and $\mu_1,\ldots,\mu_n\in\mcl{D}(\mcl{X},\alg)$ be $\alg$-valued measures.
We find an orthnormal system 
$\{p_1,\ldots,p_s\}$ 
in an RKHM (see Section~\ref{sec:c_star_alg}) that minimizes the reconstruction error as follows:
\begin{align}
&\min_{\substack{p_j:\mbox{\scriptsize ONS},\\ \blacket{p_j,p_j}: \mbox{\small rank1}}}\sum_{i=1}^n\opn{tr}\bigg(\bigg\vert\Phi(\mu_i)-\sum_{j=1}^sp_j\blacket{p_j,\Phi(\mu_i})\bigg\vert\bigg).\label{eq:reconstruction}
\end{align}
Vector $p_j$ is called the $j$-th principal axis, and $p_j\blacket{p_j,\Phi(\mu_i)}$ is called the $j$-th principal component of $\Phi(\mu_i)$.
The following proposition provides the procedure for explicitly computing $p_j$.
\begin{prop}\label{prop:pca}
Let $G\in\mathbb{C}^{mn\times mn}$ be a Hermitian matrix whose $(i,j)$-block is $\blacket{\Phi(\mu_i),\Phi(\mu_j)}\in\mat$.
Let $\sigma_1\ge\ldots\ge\sigma_{n_0}>0$ be nonzero eigenvalues of $G$ and $v_1,\ldots,v_{n_0}\in\mathbb{C}^{mn}$ be the corresponding eigenvectors.
Then, $p_j$ is represented as $p_j=\sigma_j^{-1/2}W[v_1,0,\ldots,0]$, where $W=[\Phi(\mu_1),\ldots,\Phi(\mu_n)]$.
\end{prop}

For example, if $\mu_i=\mu\rho_i$ defined in Subsection~\ref{sec:povm}, then the space spanned by $\{p_1,\ldots,p_s\}$ is interpreted as the best possible space to describe an average state of $\rho_1,\ldots,\rho_n$, 
which can be used to detect ``unusual'' states.
\subsubsection*{Numerical results}
We applied our kernel PCA in RKHM to anomaly detection for quantum states.
We generated simulation data about quantum states in the same manner as~\cite[Section III.A]{hara16}.
We generated 2500 different density matrices for a normal state,
and those for 8 erroneous states, each of which is composed of 500 matrices.
Error $1\sim 4$ corresponds to the change of phase of the density matrices, Error $5\sim 7$ corresponds to the change of amplitude, and Error $8$ corresponds to the change of both phase and amplitude.
For each erroneous state, we randomly sampled 40 matrices from the normal
states and 15 matrices from the erroneous states (in the same manner as \cite{hara16}) and computed the matrix-valued reconstruction errors with respect to the first principal components. 
We set $[k(x,y)]_{i,j}=e^{-\vert x_i-y_j\vert^2}$, which is a $c_0$-kernel (see Example~\ref{ex:pdk2}),
and set $\mu=\sum_{i=1}^{16}\delta_{\ket{\psi_i}}\ket{\psi_i}\bra{\psi_i}$, where $\ket{\psi_i}$ are constructed with products of $\ket{h}=[1,0]$, $\ket{v}=[0,1]$, $\ket{h}+e^{2/3\pi\sqrt{-1}}\ket{v}/\sqrt{2}$, and $\ket{h}+e^{4/3\pi\sqrt{-1}}\ket{v}/\sqrt{2}$.
To detect changes of both phase and amplitude, we computed those of each element of matrix-valued reconstruction error, multiplied them, then, reduced these values to a real value by the operator norm.
We compared our results with those from a previous study~\cite{hara16} (naive, ED, and GED) and those with a kernel PCA with the Hilbert--Schmidt inner product considered in~\cite{balkir14,prasenjit16}.
All the results are illustrated in Figure~\ref{fig:anomaly}.
The AUC (area under the curve) score of our method is always higher than the other methods.
This result reflects the fact each element of our matrix-valued reconstruction error corresponds to the error of each element of the density matrices, which provides sufficient information to detect the error of each element of density matrices.%
\begin{figure}[t]
 \begin{center}
  \includegraphics[scale=0.36]{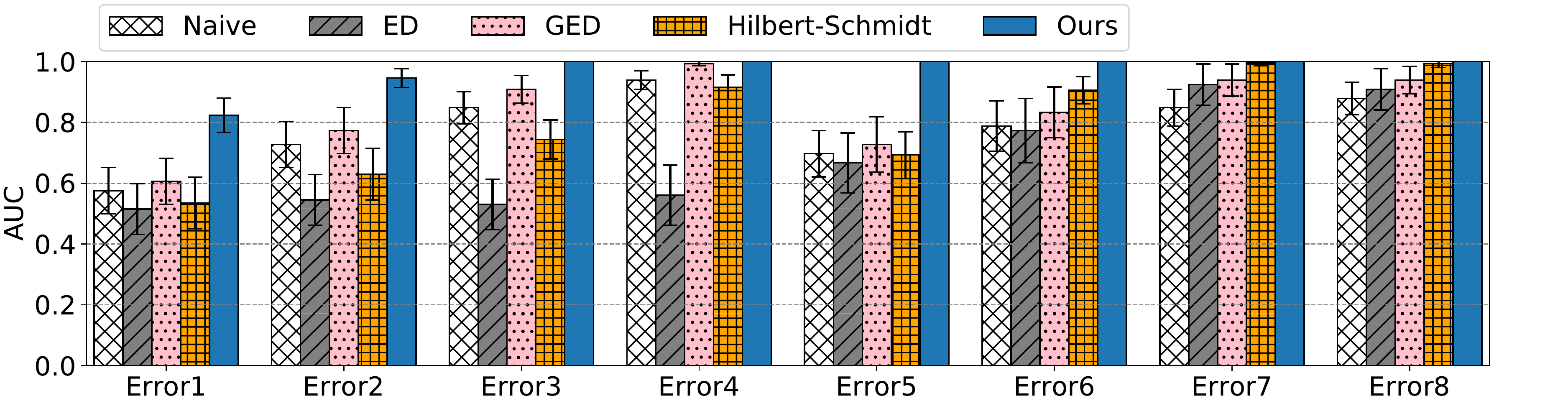}
  \caption{\small Averaged AUC scores for anomaly detection of errors $1\sim 8$ in~\cite{hara16}.}\label{fig:anomaly}
 \end{center}
\end{figure}
\subsection{Other applications}\label{sec:other}
The inner products with the proposed KME for positive operator-valued measures in Subsection~\ref{sec:povm} provide a tool for applying machine learning algorithms to inference with quantum states.
In addition, recently, random dynamical systems, which are (nonlinear) dynamical systems with random effects, have been extensively researched. 
Analyses of them by using the existing KME in RKHS have been proposed~\cite{klus17,hashimoto19}.
Our framework can generalize these results by replacing the existing KME with our KME of $\alg$-valued measures. 
For example, if we use the cross-covariance measures as $\alg$-valued measures, this enables us to analyze time-series data generated from a random dynamical system on the basis of higher-order interactions among variables.

\section{Conclusions}\label{sec:conclusion}
In this paper, we generalized the existing KME in RKHS to an embedding of a von Neumann-algebra-valued measure into an RKHM.
We derived sufficient conditions for the injectivity of the proposed KME and its connection with the universality of RKHM.
The proposed KME of von Neumann-algebra-valued measures enables us to perform probabilistic analyses reflected with higher-order interactions among variables.
Also, it generalizes the existing metric for quantum states, which can be used in applying machine learning frameworks to problems in quantum mechanics. 
Numerical results validated the advantage of the proposed methods.

\newpage
\appendix
\section*{Appendix}
We explain notations and terminologies in Section~\ref{ap:notation}.
We briefly review $\alg$-valued measures in Section~\ref{ap:review_vvmeasure} and provide detailed explanations about $c_0$-kernel in Section~\ref{ap:kernel}.
Then, we give proofs of theorems and propositions in the main body in Sections~\ref{ap:universal} and \ref{sec:proof}.

\section{Notations and terminologies}\label{ap:notation}
In this section, we describe notations and terminologies used in this paper.
Small letters denote $\alg$-valued coefficients (often by $c,d$) or vectors in $\modu$ (often by $p,q,u,v,w$).
Small Greek letters denote measures (often by $\eta,\lambda,\mu,\nu$).
Calligraphic capital letters denote sets.
The typical notations in this paper are listed in Table~\ref{tab1}.

We introduce an order in $\alg$ as follows:
For $c,d\in\alg$, let $c\le d$ mean $d-c$ is positive semi-definite.
$\le$ is a pre order in $\alg$.
And, for a subset $\mcl{S}$ of $\alg$, $a\in\alg$ is said to be an {\em upper bound} with respect to the order $\le$, if $d\le a$ for any $d\in\mcl{S}$.
Then, $c\in\alg$ is said to be a supremum of $\mcl{S}$, if $c\le a$ for any upper bound $a$ of $\mcl{S}$.

\section{$\alg$-valued measure and integral}\label{ap:review_vvmeasure}
In this section, we briefly review $\alg$-valued measure and integral 
(for further details, refer to~\cite{dinculeanu67,dinculeanu00}).
The notions of measures and Lebesgue integrals are generalized to $\alg$-valued.
\begin{defin}[$\alg$-valued measure]
Let $\mcl{X}$ be a locally compact space and $\varSigma$ be a $\sigma$-algebra on $\mcl{X}$.
\setlength{\leftmargini}{15pt}
\begin{enumerate}
\item An $\alg$-valued map $\mu:\varSigma\to\alg$ is called a {\em (countably additive) $\alg$-vaued measure} if $\mu(\bigcup_{i=1}^{\infty}E_i)=\sum_{i=1}^{\infty}\mu(E_i)$ for all countable collections $\{E_{i}\}_{i=1}^{\infty}$ of pairwise disjoint sets in $\varSigma$.
\item An $\alg$-valued measure $\mu$ is said to be finite if $\vert\mu\vert (E):=\sup\{\sum_{i=1}^n\Vert\mu(E_i)\Vert_{\alg}\mid\ n\in\mathbb{N},\ \{E_{i}\}_{i=1}^{n}\mbox{ is a finite partition of }E\in\varSigma\}<\infty$.
We call $\vert\mu\vert$ the total variation of $\mu$.
\item  An $\alg$-valued measure $\mu$ is said to be regular if for all $E\in\varSigma$ and $\epsilon>0$, there exist a compact set $K\subseteq E$ and open set $G\supseteq E$ such that $\Vert\mu(F)\Vert_{\alg}\le\epsilon$ for any $F\subseteq G\setminus K$.
The regularity corresponds to the continuity of $\alg$-valued measures.
\item An $\alg$-valued measure $\mu$ is called a {\em Borel measure} if $\varSigma=\mcl{B}$, where $\mcl{B}$ is the Borel $\sigma$-algebra on $\mcl{X}$ ($\sigma$-algebra generated by all compact subsets of $\mcl{X}$).
\end{enumerate}
The set of all $\alg$-valued finite regular Borel measures is denoted as $\measure$.
\end{defin}
\begin{table}[t]
\caption{Notation table}\label{tab1}\vspace{.3cm}
\renewcommand{\arraystretch}{1.25}
 \begin{tabularx}{\linewidth}{|c|X|}
\hline
$\mat$  & A set of all complex-valued $m\times m$ matrix\\
$\alg$ & A von Neumann-algebra\\
$\Vert\cdot\Vert_{\alg}$ & The norm in $\alg$ (For $\alg=\mat$, $\Vert c\Vert_{\mat}:=\sup_{\Vert\mathrm{d}\Vert_2=1}\Vert c\mathrm{d}\Vert_2$)\\
$\modu$ & A (right) $\alg$-module\\
$\mcl{X}$ & A locally compact Hausdorff space\\
$\measure$ & The set of all $\alg$-valued finite regular Borel measures\\
$\clch$ & The space of all continuous $\alg$-valued functions on $\mcl{X}$ vanishing at infinity\\
$k$ & An $\alg$-valued positive definite kernel\\
$\phi$ & The feature map endowed with $k$\\
$\modu_k$ & The RKHM associated with $k$\\
$\Phi$ & The proposed KME in an RKHM\\
$\vert\cdot\vert$ & The $\alg$-valued absolute value in $\modu_k$ \\
$\Vert\cdot\Vert$ & The norm in $\modu_k$ \\
$\tilde{k}$ & A complex-valued positive definite kernel\\
$\hat{\lambda}$ & The Fourier transform of an $\alg$-valued measure $\lambda$ defined as $\hat{\lambda}=\int_{\omega\in\mathbb{R}^d}e^{-\sqrt{-1}x^T\omega}d\lambda(\omega)$\\
$\opn{supp}(\lambda)$ & The support of an $\alg$-valued measure $\lambda$ defined as $\opn{supp}(\lambda):=\{x\in\mathbb{R}^d\mid\ \mbox{for any open set $U$ such that }x\in U,\ \lambda(U)\mbox{ is positive definite}\}$\\
$(\Omega,\mcl{F})$ & A measurable space \\
$P$ & A real-valued probability measure on $\Omega$\\
$X_i,Y_i$ & Real-valued random variables on $\Omega$\\
$\mu_X$ & The cross-covariance measure of $X=[X_1,\ldots,X_m]$\\
$\rho$ & A density matrix\\
$\blacket{\cdot,\cdot}_{\opn{HS}}$, $\Vert \cdot\Vert_{\opn{HS}}$& The Hilbert--Schmidt inner product and norm\\
$\gamma(\mu,\nu,\mcl{U})$ & The MMD of real-valued probability measure $\mu$ and $\nu$ with respect to a real-valued function set $\mcl{U}$\\
$\gamma_{\alg}(\mu,\nu,\mcl{U}_{\alg})$ & The proposed MMD of $\alg$-valued measure $\mu$ and $\nu$ with respect to a set of $\alg$-valued function $\mcl{U}_{\alg}$\\
$p_s$ & The $s$-th principal axis generated by kernel PCA for matrix-valued measures\\ 
\hline
 \end{tabularx}
\renewcommand{\arraystretch}{1} 
\end{table}
Similar to the Lebesgue integrals, an integral of an $\alg$-valued function with respect to an $\alg$-valued measure is defined through $\alg$-valued step functions.
\begin{defin}[Step function]\label{def:integral}
 An $\alg$-valued map $s:\mcl{X}\to\alg$ is called a {\em step function} if 
$s(x)=\sum_{i=1}^nc_i\chi_{E_i}(x)$
for some $n\in\mathbb{N}$, $c_i\in\alg$ and finite partition $\{E_{i}\}_{i=1}^{n}$ of $\mcl{X}$, where $\chi_E:\mcl{X}\to\{0,1\}$ is the indicator function for $E\in\mcl{B}$.
The set of all $\alg$-valued step functions on $\mcl{X}$ is denoted as $\mcl{S}(\mcl{X},\alg)$.
\end{defin}
\begin{defin}[Integrals of functions in $\simple$]
For $s\in\simple$ and $\mu\in\measure$, the {\em left and right integrals of $s$ with respect to $\mu$} are defined as 
\begin{equation*}
\int_{x\in\mcl{X}}s(x)d\mu(x):=\sum_{i=1}^nc_i\mu(E_i),\quad \int_{x\in\mcl{X}}d\mu(x)s(x):=\sum_{i=1}^n\mu(E_i)c_i,
\end{equation*}
respectively.
\end{defin}

As we explain below, the integrals of step functions are extended to those of ``integrable functions''.
For a real positive finite measure $\nu$, 
let $\bochner{\nu}$ be the set of all $\alg$-valued $\nu$-Bochner integrable functions on $\mcl{X}$, i.e., if $u\in\bochner{\nu}$,
there exists a sequence $\{s_i\}_{i=1}^{\infty}\subseteq\simple$ of step functions such that 
$\lim_{i\to\infty}\int_{x\in\mcl{X}}\Vert u(x)-s_i(x)\Vert_{\alg}d\nu(x)=0$~\cite[Chapter IV]{diestel84}.
Note that $u\in\bochner{\nu}$ if and only if $\int_{x\in\mcl{X}}\Vert u(x)\Vert_{\alg}d\nu(x)<\infty$, and
$\bochner{\nu}$ is a Banach $\alg$-module (i.e., a Banach space equipped with an $\alg$-module structure) with respect to the norm defined as $\Vert u\Vert_{\bochner{\nu}}=\int_{x\in\mcl{X}}\Vert u(x)\Vert_{\alg}d\nu(x)$.
\begin{defin}[Integrals of functions in $\bochner{\vert\mu\vert}$]
For $u\in\bochner{\vert\mu\vert}$, the {\em left and right integrals of $u$ with respect to $\mu$} is defined as 
\begin{equation*}
\lim_{i\to\infty}\int_{x\in\mcl{X}}d\mu(x)s_i(x),\quad
\lim_{i\to\infty}\int_{x\in\mcl{X}}s_i(x)d\mu(x)   
\end{equation*}
respectively, where $\{s_i\}_{i=1}^{\infty}\subseteq\simple$ is a sequence of step functions whose $\bochner{\nu}$-limit is $u$.
\end{defin}
Note that since $\alg$ is not commutative in general, the left and right integrals do not always coincide.

There is also a stronger notion for integrability. 
An $\alg$-valued function $u$ on $\mcl{X}$ is said to be totally measurable if it is a uniform limit of a step function, i.e., 
there exists a sequence $\{s_i\}_{i=1}^{\infty}\subseteq\mcl{S}(\mcl{X},\alg)$ of step functions such that $\lim_{i\to\infty}\sup_{x\in\mcl{X}}\Vert u(x)-s_i(x)\Vert_{\alg}=0$.
We denote by $\total$ the set of all $\alg$-valued totally measurable functions on $\mcl{X}$.
Note that if $u\in\total$, then $u\in\bochner{\vert\mu\vert}$ for any $\mu\in\measure$.

In fact, the class of continuous functions is totally measurable.
\begin{defin}[Function space $\clch$]
For a locally compact Hausdorff space $\mcl{X}$, the set of all $\alg$-valued continuous functions on $\mcl{X}$ vanishing at infinity is denoted as  $\mcl{C}_0(\mcl{X},\alg)$.
Here, an $\alg$-valued continuous function $u$ is said to vanish at infinity if the set $\{x\in\mcl{X}\mid\ \Vert u(x)\Vert_{\alg}\ge\epsilon\}$ is compact for any $\epsilon>0$.
The space $\clch$ is a Banach $\alg$-module with respect to the sup norm.
\end{defin}
\begin{prop}
The space $\mcl{C}_0(\mcl{X},\alg)$ is contained in $\mcl{T}(\mcl{X},\alg)$.
Moreover, for any real positive finite regular measure $\nu$, it is dense in $\bochner{\nu}$ with respect to $\Vert\cdot\Vert_{\bochner{\nu}}$. 
\end{prop}

\section{$c_0$-kernels}\label{ap:kernel}
In this section, we construct RKHMs that are submodules of $\clch$.
\begin{defin}[$c_0$-kernel]
Let $k:\mcl{X}\times \mcl{X}\to\alg$ be an $\alg$-valued positive definite kernel.
We call $k$ a $c_0$-kernel if $\sup_{x\in\mcl{X}}\Vert\phi(x)\Vert<\infty$ and $\phi(x)\in\clch$ for all $x\in\mcl{X}$.
\end{defin}
Note that if $k$ is a $c_0$-kernel, then $\modu_k$ is a submodule of $\clch$.
\begin{example}\label{ex:pdk1}
Let $\mcl{X}\subseteq \mathbb{C}^d$ and $k:\mcl{X}\times\mcl{X}\to\mat$ be defined as a diagonal matrix-valued kernel whose $(i,i)$-element is a complex-valued $c_0$-kernel $\tilde{k}_i$.
Then, for $x_1,\ldots,x_m\in\mcl{X}$, $c_1,\ldots,c_m\in\mat$ and $h\in\mathbb{C}^m$, $h^*(\sum_{i,j=1}^mc_i^*k(x_i,x_j)c_j)h=\sum_{i,j,l=1}^m\overline{(g_i)_l}\tilde{k}_l(x_i,x_j)(g_j)_l\ge 0$.
Thus, $k$ is an $\alg$-valued positive definite kernel and $\phi(x)\in\clch$.
Examples of complex-valued $c_0$-kernels are Gaussian, Laplacian, and $B_{2n+1}$-spline.
\end{example}
\begin{example}\label{ex:pdk2}
Assume $\mcl{X}=\mcl{Y}^m$ for some $\mcl{Y}$.
If $k:\mcl{X}\times\mcl{X}\to\mathbb{C}^{m\times m}$ is set as $[k(x,y)]_{i,j}=\tilde{k}(x_i,y_j)$ for some complex-valued $c_0$-kernel $\tilde{k}:\mcl{Y}\times \mcl{Y}\to\mathbb{C}$,
then, for $c_1,\ldots,c_s\in\mat$ and $h\in\mathbb{C}^m$, $h^*\sum_{l,l'=1}^sc_l^*k(x_l,y_{l'})c_{l'}h=\sum_{l,l'=1}^s\sum_{i,j=1}^m\overline{(g_i)_l}\tilde{k}((x_l)_i,(x_l')_j)(g_j)_{l'}\ge 0$ holds, where $g_i:=c_ih$.
Thus, $k$ is an $\alg$-valued positive definite kernel and $\phi(x)\in\clch$.
\end{example}

\section{Detailed derivation of Theorems~\ref{thm:universal_finitedim} and \ref{thm:universal}}\label{ap:universal}
Before proving Theorems~\ref{thm:universal_finitedim} and \ref{thm:universal}, we introduce
the following definitions:
\begin{defin}[$\alg$-dual]
For a Banach $\alg$-module $\modu$, the {\em $\alg$-dual of $\modu$} is defined as $\modu':=\{f:\modu\to\alg\mid\ f\mbox{ is bounded and $\alg$-linear}\}$.
\end{defin}
Note that for a right Banach $\alg$-module $\modu$, $\modu'$ is a left Banach $\alg$-module.
\begin{defin}[Orthogonal complement]
For an $\alg$-submodule $\modu_0$ of a Banach $\alg$-module $\modu$, the {\em orthogonal complement of $\modu_0$} is defined as a closed submodule $\modu_0^{\perp}:=\bigcap_{u\in\modu_0}\{f\in\modu'\mid\ f(u)=0\}$ of $\modu'$.
In addition, for an $\alg$-submodule $\mcl{N}_0$ of $\modu'$, the {\em the orthogonal complement of $\mcl{N}_0$} is defined as a closed submodule $\mcl{N}_0^{\perp}:=\bigcap_{f\in\mcl{N}_0}\{u\in\modu\mid\ f(u)=0\}$ of $\modu$.
\end{defin}
Note that for a von Neumann-algebra $\alg$ and Hilbert $\alg$-module $\modu$, by Proposition~\ref{thm:riesz}, $\modu'$ and $\modu$ are isomorphic.

The following lemma shows a connection between an orthogonal complement and the density property.
\begin{lemma}\label{lem:orthocompequiv2}
For a Banach $\alg$-module $\modu$ and its submodule $\modu_0$,  
$\modu_0^{\perp}=\{0\}$ if 
$\modu_0$ is dense in $\modu$.
\end{lemma}
\begin{proof}
We first show $\overline{\modu_0}\subseteq (\modu_0^{\perp})^{\perp}$.
Let $u\in\modu_0$. By the definition of orthogonal complements, $u\in(\modu_0^{\perp})^{\perp}$.
Since $(\modu_0^{\perp})^{\perp}$ is closed, $\overline{\modu_0}\subseteq (\modu_0^{\perp})^{\perp}$.
If $\modu_0$ is dense in $\modu$, $\modu\subseteq (\modu_0^{\perp})^{\perp}$ holds, which means $\modu_0^{\perp}=\{0\}$.
\end{proof}

Let $\regular$ be the set of all real positive-valued regular measures, and $\bdmeasure{\nu}$ the set of all finite regular Borel $\alg$-valued measures $\mu$ whose total variations are dominated by $\nu\in\regular$ (i.e., $\vert\mu\vert\le\nu$).
We apply the following representation theorem to derive Theorem~\ref{thm:universal}.
\begin{prop}
\label{prop:represention}
For $\nu\in\regular$, 
there exists an isomorphism between $\bdmeasure{\nu}$ and $\bochner{\nu}'$.
\end{prop}
\begin{proof}
For $\mu\in\bdmeasure{\nu}$ and $u\in\bochner{\nu}$, we have
\begin{equation*}
\bigg\Vert\int_{x\in\mcl{X}}d\mu(x)u(x)\bigg\Vert_{\alg}
\le\int_{x\in\mcl{X}}\Vert u(x)\Vert_{\alg}d\vert\mu\vert(x)
\le\int_{x\in\mcl{X}}\Vert u(x)\Vert_{\alg}d\nu(x).
\end{equation*}
Thus, we define $h:\bdmeasure{\nu}\to\bochner{\nu}'$ as $\mu\mapsto (u\mapsto\int_{x\in\mcl{X}}d\mu(x)u(x))$.

Meanwhile, for $f\in\bochner{\nu}'$ and $E\in\mcl{B}$, we have
\begin{equation*}
\Vert f(\chi_E 1_{\alg})\Vert_{\alg}\le C\int_{x\in\mcl{X}}\Vert \chi_E 1_{\alg}\Vert_{\alg}d\nu(x)
=C\nu(E),
\end{equation*}
for some $C>0$ since $f$ is bounded.
Here, $\chi_E$ is an indicator function for a Borel set $E$. 
Thus, we define $h':\bochner{\nu}'\to\bdmeasure{\nu}$ as $f\mapsto(E\mapsto f(\chi_E 1_{\alg}))$.

By the definitions of $h$ and $h'$, $h(h'(f))(s)=f(s)$ holds for $s\in\simple$.
Since $\simple$ is dense in $\bochner{\nu}$, $h(h'(f))(u)=f(u)$ holds for $u\in\bochner{\nu}$.
Moreover, $h'(h(\mu))(E)=\mu(E)$ holds for $E\in\mcl{B}$.
Therefore, $\bdmeasure{\nu}$ and $\bochner{\nu}'$ are isomorphic.
\end{proof}

\begin{proof}[Proof of Theorem~\ref{thm:universal}]
Assume $\modu_k$ is dense in $\clch$.
Since $\clch$ is dense in $\bochner{\nu}$ for any $\nu\in\regular$, $\modu_k$ is dense in $\bochner{\nu}$ for any $\nu\in\regular$.
By Proposition~\ref{lem:orthocompequiv2}, $\modu_k^{\perp}=\{0\}$ holds.
Let $\mu\in\measure$.
There exists $\nu\in\regular$ such that $\mu\in\bdmeasure{\nu}$.
By Proposition~\ref{prop:represention}, if $\int_{x\in\mcl{X}}d\mu(x)u(x)=0$ for any $u\in\modu_k$, $\mu=0$.
Since $\int_{x\in\mcl{X}}d\mu(x)u(x)=\blacket{u,\Phi(\mu)}$, $\int_{x\in\mcl{X}}d\mu(x)u(x)=0$ means $\Phi(\mu)=0$.
Therefore, by Lemma~\ref{lem:injective_equiv}, $\Phi$ is injective.
\end{proof}

For the case of $\alg=\mat$, we apply the following extension theorem to derive the converse of Theorem~\ref{thm:universal}.
\begin{prop}[c.f. Theorem in~\cite{helemskii94}]
\label{prop:hahn_banach}
Let $\alg=\mat$. 
Let $\modu$ be a Banach $\alg$-module, $\modu_0$ be a closed submodule of $\modu$, and $f_0:\modu_0\to\alg$ be a bounded $\alg$-linear map.
Then, there exists a bounded $\alg$-linear map $f:\modu\to\alg$ that extends $f_0$ (i.e., $f(u)=f_0(u)$ for $u\in\modu_0$).
\end{prop}
\begin{proof}
Von Neumann-algebra $\alg$ itself is regarded as an $\alg$-module and is normal.
Also, $\mat$ is Connes injective.
By Theorem in~\cite{helemskii94}, $\alg$ is an injective object in the category of Banach $\alg$-module.
The statement is derived by the definition of injective objects in category theory. 
\end{proof}

We derive the following lemma and proposition by Proposition~\ref{prop:hahn_banach}.
\begin{lemma}\label{lem:hahn_banach2}
Let $\alg=\mat$.
Let $\modu$ be a Banach $\alg$-module and $\modu_0$ be a closed submodule of $\modu$.
For $u_1\in\modu\setminus \modu_0$, there exists a bounded $\alg$-linear map $f:\modu\to\alg$ such that $f(u_0)=0$ for $u_0\in\modu_0$ and $f(u_1)\neq 0$.
\end{lemma}
\begin{proof}
Let $q:\modu\to\modu/\modu_0$ be the quotient map to $\modu/\modu_0$, and $\,\mcl{U}_1:=\{q(u_1)c\mid\ c\in\alg\}$.
Note that $\modu/\modu_0$ is a Banach $\alg$-module and $\,\mcl{U}_1$ is its closed submodule.
Let $\mcl{V}:=\{c\in\alg\mid\ q(u_1)c=0\}$, which is a closed subspace of $\alg$.
Since $\mcl{V}$ is orthogonally complemented~\cite[Proposition 2.5.4]{manuilov00}, 
$\alg$ is decomposed into $\alg=\mcl{V}+\mcl{V}^{\perp}$.
Let $p:\alg\to\mcl{V}^{\perp}$ be the projection onto $\mcl{V}^{\perp}$ and
$f_0:\mcl{U}_1\to\alg$ defined as $q(u_1)c\mapsto p(c)$.
Since $p$ is $\alg$-linear, $f_0$ is also $\alg$-linear.
Also, for $c\in\alg$, we have
\begin{align*}
\Vert q(u_1)c\Vert&=\Vert q(u_1)(c_1+c_2)\Vert 
=\Vert q(u_1)c_1\Vert\\
&\ge \inf_{d\in\mcl{V}^{\perp},\Vert d\Vert_{\alg}=1}\Vert q(u_1)d\Vert \Vert c_1\Vert_{\alg}
=\inf_{d\in\mcl{V}^{\perp},\Vert d\Vert_{\alg}=1}\Vert q(u_1)d\Vert \Vert p(c)\Vert_{\alg},
\end{align*}
where $c_1=p(c)$ and $c_2=c_1-p(c)$.
Since $\inf_{d\in\mcl{V}^{\perp},\Vert d\Vert_{\alg}=1}\Vert q(u_1)d\Vert \Vert p(c)\Vert_{\alg}>0$, $f_0$ is bounded.
By Proposition~\ref{prop:hahn_banach}, $f_0$ is extended to a bounded $\alg$-linear map $f_1:\modu/\modu_0\to\alg$.
Setting $f:=f_1\circ q$ completes the proof of the lemma.
\end{proof}
Then we prove the converse of Lemma \ref{lem:orthocompequiv2}.
\begin{prop}\label{prop:orthcompequiv}
Let $\alg=\mat$.
For a Banach $\alg$-module $\modu$ and its submodule $\modu_0$, 
$\modu_0$ is dense in $\modu$ if $\modu_0^{\perp}=\{0\}$.
\end{prop}
\begin{proof}
Assume $u\notin\overline{\modu_0}$. 
We show $\overline{\modu_0}\supseteq (\modu_0^{\perp})^{\perp}$
By Lemma~\ref{lem:hahn_banach2}, there exists $f\in\modu'$ such that $f(u)\neq 0$ and $f(u_0)=0$ for any $u_0\in\overline{\modu_0}$.
Thus, $u\notin (\modu_0^{\perp})^{\perp}$.
As a result, $\overline{\modu_0}\supseteq (\modu_0^{\perp})^{\perp}$.
Therefore, if $\modu_0^{\perp}=\{0\}$, then $\overline{\modu_0}\supseteq\modu$, which implies $\modu_0$ is dense in $\modu$.
\end{proof}

In the case of $\alg=\mat$, a generalization of the Riesz-Markov representation theorem with respect to $\clch$ holds.
\begin{prop}\label{prop:representation_finitedim}
Let $\alg=\mat$.
There exists an isomorphism between $\measure$ and $\clch'$.
\end{prop}
\begin{proof}
We define $h:\measure\to\clch$ in the same manner as Proposition~\ref{prop:represention}.
For $f\in\clch'$, let $f_{i,j}\in\mcl{C}_0(\mcl{X},\mathbb{C})'$ be defined as $f_{i,j}(u)=(f(u1_{\alg}))_{i,j}$ for $u\in\mcl{C}_0(\mcl{X},\mathbb{C})$.
Then, by the Riesz--Markov representation theorem for complex-valued measure, there exists a unique finite complex-valued regular measure $\mu_{i,j}$ such that $f_{i,j}(u)=\int_{x\in\mcl{X}}u(x)d\mu_{i,j}(x)$.
Let $\mu(E):=[\mu_{i,j}(E)]_{i,j}$ for $E\in\mcl{B}$.
Then, $\mu\in\measure$, and we have
\begin{align*}
f(u)&=f\bigg(\sum_{l,l'=1}^m u_{l,l'}e_{l,l'}\bigg)
=\sum_{l,l'=1}^m[f_{i,j}(u_{l,l'})]_{i,j}e_{l,l'}\\
&=\sum_{l,l'=1}^m\bigg[\int_{x\in\mcl{X}}u_{l,l'}(x)d\mu_{i,j}(x)\bigg]_{i,j}e_{l,l'}
=\int_{x\in\mcl{X}}d\mu(x)u(x),
\end{align*}
where $e_{i,j}$ is an $m\times m$ matrix whose $(i,j)$-element is $1$ and all the other elements are $0$.
Therefore, if we define $h':\clch'\to\measure$ as $f\mapsto \mu$, $h'$ is the inverse of $h$, which completes the proof of the proposition.
\end{proof}

As a result, we derive Theorem~\ref{thm:universal_finitedim} as follows:
\begin{proof}[Proof of Theorem~\ref{thm:universal_finitedim}]
For $\mu\in\measure$, $\Phi(\mu)=0$ is equivalent to $\int_{x\in\mcl{X}}d\mu^*(x)u(x)=\blacket{\Phi(\mu),u}_k=0$ for any $u\in\modu_k$.
Thus, by Proposition~\ref{prop:representation_finitedim}, ``$\Phi(\mu)=0\Rightarrow \mu=0$'' is equivalent to ``$f\in\clch'$, $f(u)=0$ for any $u\in\modu_k$ $\Rightarrow$ $f=0$''.
By the definition of $\modu_k^{\perp}$ and Proposition~\ref{prop:orthcompequiv}, $\modu_k$ is dense in $\clch$.
\end{proof}

\section{Proofs}\label{sec:proof}
\subsection*{Proof of Theorem~\ref{thm:kme}}
We use the Cauchy-Schwarz inequality for a Hilbert $\alg$-module $\modu$.
\begin{lemma}[Cauchy-Schwarz inequality~\cite{lance95}]\label{lem:c-s}
For $u,v\in\modu$, the following inequality holds:
\begin{equation*}
\vert\blacket{u,v}\vert^2\le\Vert u\Vert^2\blacket{v,v}.
\end{equation*}
\end{lemma}

Let $L_{\mu}:\modu_k\to\alg$ be an $\alg$-linear map defined as $L_{\mu}v:=\int_{x\in\mcl{X}}d\mu^*(x)v(x)$.
The following inequalities are derived by the reproducing property~\eqref{eq:reproducing}, and  Lemma~\ref{lem:c-s}:
\begin{align}
 \Vert L_{\mu}v\Vert_{\alg}
&\le\int_{x\in\mcl{X}}\Vert v(x)\Vert_{\alg}d\vert \mu\vert(x)
=\int_{x\in\mcl{X}}\Vert \blacket{\phi(x),v}\Vert_{\alg}d\vert \mu\vert(x)\nn\\
&\le \Vert v\Vert\int_{x\in\mcl{X}}\Vert \phi(x)\Vert d\vert \mu\vert(x)
\le \vert\mu\vert(\mcl{X})\Vert v\Vert\sup_{x\in\mcl{X}}\Vert\phi(x)\Vert,\label{eq:bounded}
\end{align}
where the first inequality is easily checked for a step function $s(x):=\sum_{i=1}^nc_i\chi_{E_i}(x)$ as follows and thus, it holds for any totally measurable functions:
\begin{align*}
\bigg\Vert \int_{x\in\mcl{X}}d\mu^*(x)s(x)\bigg\Vert_{\alg}
&=\Vert \sum_{i=1}^n\mu(E_i)^*c_i\Vert_{\alg}
\le \sum_{i=1}^n\Vert \mu(E_i)\Vert_{\alg}\Vert c_i\Vert_{\alg}\\
&\le \sum_{i=1}^n\vert \mu\vert(E_i)\Vert c_i\Vert_{\alg}
=\int_{x\in\mcl{X}}\Vert s(x)\Vert_{\alg}d\vert \mu\vert(x),
\end{align*}
Since both $\vert{\mu}\vert(\mcl{X})$ and $\sup_{x\in\mcl{X}}\Vert\phi(x)\Vert$ are finite, inequality~\eqref{eq:bounded} means $L_{\mu}$ is bounded.
Thus, by the Riesz representation theorem for Hilbert $\alg$ modules (Theorem~\ref{thm:riesz}), there exists $u_{\mu}\in\modu_k$ such that $L_{\mu}v=\blacket{u_{\mu},v}$.
By setting $v=\phi(y)$, we have $u_{\mu}(y)=L_{\mu}\phi(y)^*=\int_{x\in\mcl{X}}k(y,x)d\mu(x)$ for $y\in\mcl{X}$.
Therefore, $\Phi(\mu)=u_{\mu}\in\modu_k$ and $\blacket{\Phi(\mu),v}=\int_{x\in\mcl{X}}d\mu^*(x)v(x)$.

\subsection*{Proof of Theorem~\ref{thm:characteristic}}
The following lemma is used to show the injectivity of $\Phi$.
\begin{lemma}\label{lem:injective_equiv}
 $\Phi:\mcl{D}(\mcl{X},\alg)\to\modu_k$ is injective if and only if $\blacket{\Phi(\mu),\Phi(\mu)}\neq 0$ for any nonzero $\mu\in\mcl{D}(\mcl{X},\alg)$.
\end{lemma}
\begin{proof}
($\Rightarrow$) Suppose there exists a nonzero $\mu\in\mcl{D}(\mcl{X},\alg)$ such that $\blacket{\Phi(\mu),\Phi(\mu)}=0$.
Then, $\Phi(\mu)=\Phi(0)=0$ holds, and thus, $\Phi$ is not injective.

($\Leftarrow$) Suppose $\Phi$ is not injective. Then, there exist $\mu,\nu\in\mcl{D}(\mcl{X},\alg)$ such that $\Phi(\mu)=\Phi(\nu)$ and $\mu\neq\nu$, which implies $\Phi(\mu-\nu)=0$ and $\mu-\nu\neq 0$.
\end{proof}

\begin{proof}[Proof of Theorem~\ref{thm:characteristic}]
Let $\mu\in\mcl{D}(\mcl{X},\alg)$, $\mu\neq 0$.
We have
\begin{align*}
 \blacket{\Phi(\mu),\Phi(\mu)}&=\int_{x\in\mathbb{R}^d}\int_{y\in\mathbb{R}^d}d\mu^*(x)k(x,y)d\mu(y)\\
&=\int_{x\in\mathbb{R}^d}\int_{y\in\mathbb{R}^d}d\mu^*(x)\int_{\omega\in\mathbb{R}^d}e^{-\sqrt{-1}(y-x)^T\omega}d\lambda(\omega)d\mu(y)\\
&=\int_{\omega\in\mathbb{R}^d}\int_{x\in\mathbb{R}^d}e^{\sqrt{-1}x^T\omega}d\mu^*(x)d\lambda(\omega)\int_{y\in\mathbb{R}^d}e^{-\sqrt{-1}y^T\omega}d\mu(y)\\
&=\int_{\omega\in\mathbb{R}^d}\hat{\mu}(\omega)^*d\lambda(\omega)\hat{\mu}(\omega).
\end{align*}
Since $\mu$ is a countably additive Borel measure,
for $\mu\neq 0$, $\hat{\mu}\neq 0$ holds.
In addition, by the assumption, $\opn{supp}(\lambda)=\mathbb{R}^d$ holds.
As a result, $\int_{\omega\in\mathbb{R}^d}\hat{\mu}(\omega)^*d\lambda(\omega)\hat{\mu}(\omega)\neq 0$ holds.
By Lemma~\ref{lem:injective_equiv}, $\Phi$ is injective.
\end{proof}

\subsection*{Proof of Theorem~\ref{thm:characteristic2}}
Let $\mu\in\mcl{D}(\mcl{X},\alg)$, $\mu\neq 0$.
We have
\begin{align}
\blacket{\Phi(\mu),\Phi(\mu)}
&=\int_{x\in\mathbb{R}^d}\int_{y\in\mathbb{R}^d}d\mu^*(x)k(x,y)d\mu(y)\nn\\
&=\int_{x\in\mathbb{R}^d}\int_{y\in\mathbb{R}^d}d\mu^*(x)\int_{t\in[0,\infty)}e^{-t\Vert x-y\Vert^2}d\eta(t)d\mu(y)\nn\\
&=\int_{x\in\mathbb{R}^d}\int_{y\in\mathbb{R}^d}d\mu^*(x)\int_{t\in[0,\infty)}\frac{1}{(2t)^{d/2}}\int_{\omega\in\mathbb{R}^d}e^{-\sqrt{-1}(y-x)^T\omega-\frac{\Vert\omega\Vert^2}{4t}}d\omega d\eta(t)d\mu(y)\nn\\
&=\int_{\omega\in\mathbb{R}^d}\hat{\mu}(\omega)^*\int_{t\in[0,\infty)}\frac{1}{(2t)^{d/2}}e^{\frac{-\Vert\omega\Vert^2}{4t}}d\eta(t)\hat{\mu}(\omega)d\omega,\label{eq:radial}
\end{align}
where we applied a formula $e^{-t\Vert x\Vert^2}={(2t)^{-d/2}}\int_{\omega\in\mathbb{R}^d}e^{-\sqrt{-1}x^T\omega-\Vert\omega\Vert^2/(4t)}d\omega$ in the third equality.
Since $\mu$ is a countably additive Borel measure,
for $\mu\neq 0$, $\hat{\mu}\neq 0$ holds.
In addition, since $\opn{supp}(\eta)\neq\{0\}$ holds, $\int_{t\in[0,\infty)}(2t)^{-d/2}e^{-\Vert\omega\Vert^2/(4t)}d\eta(t)$ is positive definite.
As a result, the last formula in Eq.~\eqref{eq:radial} is nonzero. 
By Lemma~\ref{lem:injective_equiv}, $\Phi$ is injective.

\subsection*{Proof of Theorem~\ref{prop:cross-covariance}}
The inner product between $\Phi(\mu_X)$ and $\Phi(\mu_Y)$ is calculated as follows:
\begin{align*}
&\blacket{\Phi(\mu_X),\Phi(\mu_Y)}_k\\
&\qquad=\int_{x\in\mcl{X}^2}\int_{y\in\mcl{X}^2}d\mu_X^*(x)k(x,y)d\mu_Y(y)\\
&\qquad=\bigg[\sum_{l=1}^m\int_{x\in\mcl{X}^2}\int_{y\in\mcl{X}^2}d(X_l,X_i)_*P(x)\tilde{k}_1(x_1,y_1)\tilde{k}_2(x_2,y_2)d(Y_l,Y_j)_*P(y)\bigg]_{i,j}\\
&\qquad=\bigg[\sum_{l=1}^m\int_{\omega\in\Omega}\int_{\eta\in\Omega}dP(\omega)\blacket{\tilde{\psi}_1(X_l(\omega)),\tilde{\psi}_1(Y_l(\eta))}\blacket{\tilde{\psi}_2(X_i(\omega)),\tilde{\psi}_2(Y_j(\eta))}dP(\eta)\bigg]_{i,j}\\
&\qquad=\bigg[\sum_{l=1}^m\blacket{\Sigma_{X_l,X_i},\Sigma_{Y_l,Y_j}}_{\opn{HS}}\bigg]_{i,j}
\end{align*}
Since $\Sigma_{X_i,X_j}$ is a Hilbert--Schmidt operator for any $i,j\in\{1,\ldots,m\}$, $\Sigma_X$ is also a Hilbert--Schmidt operator, and we have
\begin{equation*}
\opn{tr}(\blacket{\Phi(\mu_X),\Phi(\mu_Y)})=\sum_{i=1}^m\sum_{l=1}^m\opn{tr}(\Sigma_{X_l,X_i}^*\Sigma_{Y_l,Y_i})=\sum_{i=1}^m\opn{tr}(\Sigma_{X}^*\Sigma_{Y})=\blacket{\Sigma_{X},\Sigma_{Y}}_{\opn{HS}}.
\end{equation*}
As a result, $\opn{tr}(\vert \Phi(\mu_X)-\Phi(\mu_Y)\vert)=\Vert \Sigma_X-\Sigma_Y\Vert_{\opn{HS}}$ holds.

\subsection*{Proof of Theorem~\ref{prop:inprod_equiv}}
Let $M_i=\ket{\psi_i}\bra{\psi_i}$ for $i=1,\ldots,m$.
The inner product between $\Phi(\mu\rho_1)$ and $\Phi(\mu\rho_2)$ is calculated as follows:
\begin{align*}
 \blacket{\Phi(\mu\rho_1),\Phi(\mu\rho_2)}&=\int_{x\in\mcl{X}}\int_{y\in\mcl{X}}\rho_1^*\mu^*(x)k(x,y)\mu\rho_2(y)
=\sum_{i,j=1}^m\rho_1^*M_ik(\ket{\psi_i},\ket{\psi_{j}})M_j\rho_2.
\end{align*}
Since $k(\ket{\psi_i},\ket{\psi_j})=M_iM_j$ and $\{\ket{\psi_1},\ldots,\ket{\psi_m}\}$ is orthonormal, equality $ \blacket{\Phi(\mu\rho_1),\Phi(\mu\rho_2)}=\sum_{i=1}^m\rho_1^*M_i\rho_2$ holds.
By using the equality $\sum_{i=1}^mM_i=I$, $\opn{tr}(\sum_{i=1}^m\rho_1^*M_i\rho_2)=\opn{tr}(\sum_{i=1}^mM_i\rho_2\rho_1^*)=\opn{tr}(\rho_2\rho_1^*)$ hold, which completes the proof of the theorem.

\subsection*{Proof of Proposition~\ref{prop:mmd}}

By Lemma~\ref{lem:c-s}, we have
 \begin{align*}
  \bigg\vert \int_{x\in\mcl{X}}d\mu^*u(x)-\int_{x\in\mcl{X}}d\nu^*u(x)\bigg\vert_{\alg}=\vert \blacket{\Phi(\mu-\nu),u}\vert_{\alg}
\le \Vert u\Vert\vert \Phi(\mu-\nu)\vert\le\vert\Phi(\mu-\nu)\vert,
 \end{align*}
for any $u\in\modu_k$ such that $\Vert u\Vert\le 1$.
Let $\epsilon>0$.
We put $v=\Phi(\mu-\nu)$ and $u_{\epsilon}=v(\vert v\vert+\epsilon 1_{\alg})^{-1}$.
Then, $\vert v\vert^2\le (\vert v\vert+\epsilon 1_{\alg})^2$ holds.
By multiplying $(\vert v\vert+\epsilon 1_{\alg})^{-1}$ on the both sides, we have $\vert u_{\epsilon}\vert^2\le 1_{\alg}$.
Thus, $\Vert u_{\epsilon}\Vert\le 1$.
In addition, the following is derived:
\begin{align*}
(\vert v\vert+\epsilon 1_{\alg})\vert v\vert^2-(\vert v\vert^2-\epsilon^21_{\alg})(\vert v\vert+\epsilon 1_{\alg})
=\epsilon^2(\vert v\vert+\epsilon 1_{\alg})\ge 0.
\end{align*}
By multiplying $(\vert v\vert+\epsilon 1_{\alg})^{-1}$ on the both sides, we have $\vert\blacket{v,u_{\epsilon}}\vert_{\alg}+\epsilon 1_{\alg}-\vert v\vert\ge 0$,
which implies $\Vert \vert v\vert-\blacket{v,u_{\epsilon}}\Vert_{\alg}\le\epsilon$, and $\lim_{\epsilon\to 0}\blacket{v,u_{\epsilon}}=\vert v\vert$.
Since $\blacket{v,u_{\epsilon}}\le d$ for any upper bound $d$ of $\{\vert\blacket{v,u}\vert\mid\ \Vert u\Vert\le 1\}$, $\vert v\vert\le d$ holds.
As a result, $\vert v\vert=\vert \Phi(\mu-\nu)\vert$ is the supremum of $\vert \int_{x\in\mcl{X}}d\mu^*u(x)-\int_{x\in\mcl{X}}d\nu^*u(x)\vert_{\alg}$.

\subsection*{Proof of Proposition~\ref{prop:pca}}
The objective function in Eq.~\ref{eq:reconstruction} is transformed as follows:
\begin{align*}
&\sum_{i=1}^n\opn{tr}\bigg(\bigg\vert\Phi(\mu_i)-\sum_{j=1}^sp_j\blacket{p_j,\Phi(\mu_i})\bigg\vert^2\bigg)\\
&\qquad=\sum_{i=1}^n\opn{tr}(\vert\Phi(\mu_i)\vert^2)-\opn{tr}\bigg(\sum_{i=1}^n\sum_{j=1}^s\vert\blacket{p_j,\Phi(\mu_i})\vert^2\bigg).
\end{align*}
Thus, minimization problem~\eqref{eq:reconstruction} is equal to the following maximization problem:
\begin{equation}
\max_{\substack{p_j:\mbox{\scriptsize ONS},\\ \blacket{p_j,p_j}: \mbox{\small rank1}}}\opn{tr}\bigg(\sum_{i=1}^n\sum_{j=1}^s\vert\blacket{p_j,\Phi(\mu_i})\vert^2\bigg).\label{eq:reconstruction_max}
\end{equation}

Since $\alg$ is a von Neumann-Algebra, $\opn{Span}\{\Phi(\mu_1),\ldots,\Phi(\mu_n)\}$ is orthogonally complemented. 
Thus, $p_j$ is represented as $p_j=p_j^{\parallel}+p_j^{\perp}$ for $p_j^{\parallel}\in\opn{Span}\{\Phi(\mu_1),\ldots,\Phi(\mu_n)\}$ and $p_j^{\perp}\perp\opn{Span}\{\Phi(\mu_1),\ldots,\Phi(\mu_n)\}^{\perp}$.
Let $p_j^{\parallel}=\sum_{i=1}^n\Phi(\mu_i)c_{i,j}$ for some $c_{i,j}\in\alg$.
By substituting $p_j=\sum_{i=1}^n\Phi(\mu_i)c_{i,j}+p_j^{\perp}$ to the objective function of maximization problem~\eqref{eq:reconstruction_max}, we have
\begin{align*}
\sum_{j=1}^s\opn{tr}\bigg(\sum_{i=1}^n\vert\blacket{p_j,\Phi(\mu_i})\vert^2\bigg)
&=\sum_{j=1}^s\opn{tr}\bigg(\sum_{i=1}^n\sum_{l,l'=1}^nc_{l,j}^*\blacket{\Phi(\mu_l),\Phi(\mu_i)}\blacket{\Phi(\mu_i),\Phi(\mu_{l'})}c_{l',j}\bigg)\\
&=\sum_{j=1}^s\opn{tr}(c_j^*G^2c_j)=\sum_{j=1}^s\opn{tr}\Big(\big(\sqrt{G}c_j\big)^*G\big(\sqrt{G}c_j\big)\Big),
\end{align*}
where $c_j=[c_{1,j},\ldots,c_{n,j}]^T$.
Since $\{p_1,\ldots,p_s\}$ is an ONS and $\blacket{p_i,p_i}$ is rank-one, $c_i^*Gc_j=0$ for $i\neq j$ and $c_i^*Gc_i$ is a rank-one projection.
Therefore, any $c_j$ that satisfies $\sqrt{G}c_j=[v_j,0,\ldots,0]$ attains the maximum of problem~\eqref{eq:reconstruction_max}.
Thus, $p_j=\sum_{i=1}^n\Phi(\mu_i)c_{i,j}$, where $c_j=\lambda_j^{-1/2}[v_j,0,\ldots,0]$ is a solution of maximization problem~\eqref{eq:reconstruction_max}, thus, that of minimization problem~\eqref{eq:reconstruction}.

\end{document}